\documentclass[11pt]{article}

\usepackage[final]{acl}

\usepackage{times}
\usepackage{latexsym}
\usepackage[T1]{fontenc}
\usepackage[utf8]{inputenc}
\usepackage{microtype}
\usepackage{inconsolata}
\usepackage{amsmath,amssymb,amsfonts,amsthm}
\usepackage{booktabs}
\usepackage{algorithm}
\usepackage[noend]{algpseudocode}
\usepackage{multirow}
\usepackage{graphicx}
\usepackage{xcolor}
\usepackage{enumitem}
\usepackage{url}
\newtheorem{definition}{Definition}
\newtheorem{theorem}{Theorem}
\newtheorem{proposition}{Proposition}

\usepackage{xcolor}
\usepackage{pifont}

\IfFileExists{fontawesome5.sty}{%
  \usepackage{fontawesome5}%
  \newcommand{\codeicon}{\faGithub~}%
}{%
  \newcommand{\codeicon}{}%
}

\newcommand{\methodname}{\textsc{CovCal}}
\newcommand{\calX}{\mathcal{X}}

\newcommand{\calS}{\mathcal{S}}
\newcommand{\calT}{\mathcal{T}}
\newcommand{\calD}{\mathcal{D}}
\newcommand{\calC}{\mathcal{C}}
\newcommand{\calO}{\mathcal{O}}
\newcommand{\ind}{\mathbb{I}}

\newcommand{\typ}{\mathrm{typ}}
\newcommand{\prf}{\mathrm{prf}}
\newcommand{\unres}{\mathrm{unres}}

\title{
Risk-Controlled Lean-as-Judge for \\Natural-Language Mathematical Reasoning}

\author{Pauline Bourigault \\
Imperial College London
\And Xiaotong Ji \\
Huawei Noah's Ark Lab
\And Matthieu Zimmer \\
Huawei Noah's Ark Lab\\
\AND 
Rasul Tutunov \\
Huawei Noah's Ark Lab
\And Haitham Bou-Ammar 
\\
Huawei Noah's Ark Lab\\ 
UCL Centre for AI
}

\begin{document}
\maketitle

\begin{abstract}
Lean is increasingly used to judge natural-language mathematical answers, but its signal is partial: many answers never formalize, and a failed proof may reflect an ill-typed statement or a missing library fact, not a wrong answer. On MATH-500 we show this signal is (i) sharply coverage-dependent, that is the proof-winning answer is correct $96\%$ of the time at high proved coverage but $20\%$ at low, and (ii) sparse and often unfaithful: a $7$B autoformalizer proves a class for only $28\%$ of problems, and a manual audit finds only $\approx\!43\%$ of those proofs faithful. We propose \methodname{}, a selector over Lean-trace diagnostics that certifies a finite-sample selective-risk bound on accepted answers or abstains, under two regimes (a conservative Bonferroni bound and a tighter dev-then-cal rule). Feasibility depends on autoformalization coverage: with the $7$B formalizer the signal is too sparse and Bonferroni abstains on all $20$ bootstrap partitions, whereas a prover-specialized formalizer reaches $79\%$ coverage and flips it to feasible on $17$ of $20$, accepting $\approx\!48\%$ of problems at $0.98$ accepted accuracy. Since self-consistency alone is already $91\%$ accurate, our contribution is a precise account of when, and with which formalizer, a partial formal signal can be trusted under risk control.
\end{abstract}

\begin{figure*}[h]
\centering
\includegraphics[width=\textwidth]{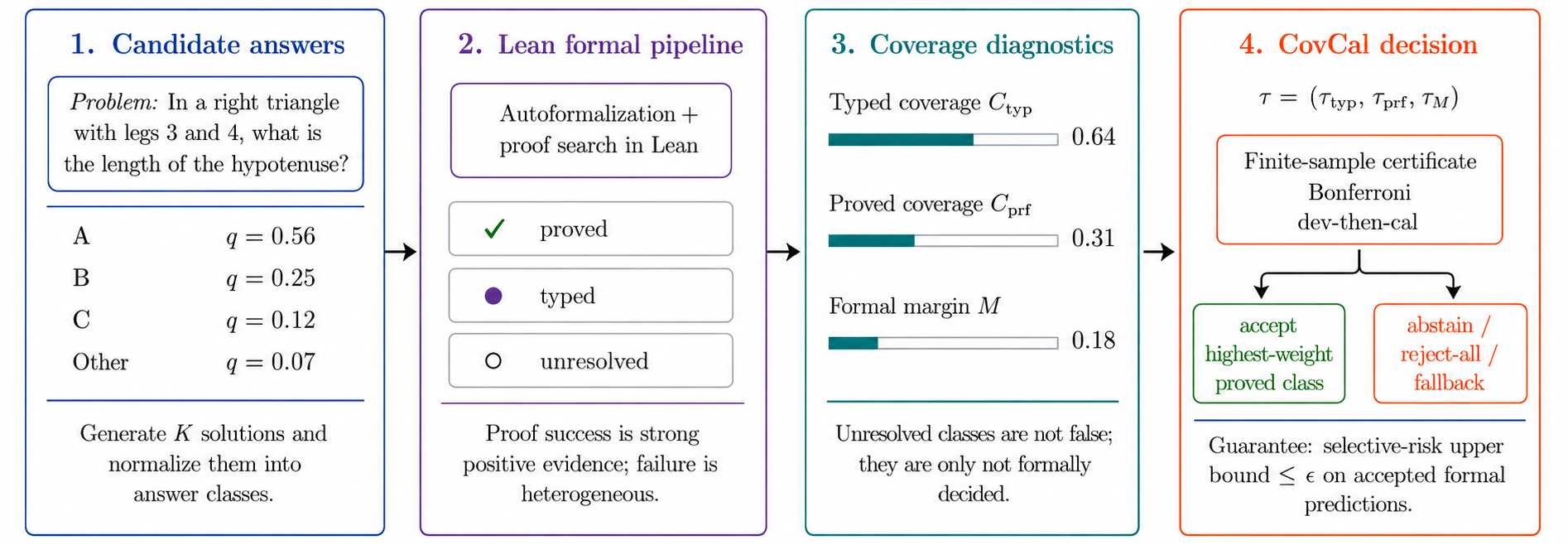}
\caption{\methodname{} treats Lean as a partial-observation judge: trust formal evidence when coverage is sufficient, and abstain when coverage is too low. The method couples Lean-based formal traces with finite-sample selective prediction.}
\label{fig:overview}
\end{figure*}

\section{Introduction}

A Lean-kernel-checked proof 
is strong positive evidence for the formal statement that Lean has checked.
Recent NLP systems use this signal by translating natural-language question--answer pairs into Lean, attempting proof search, and preferring candidates whose formal proofs succeed \citep{yao2025fans,liu2025safe}. Our question is complementary: \emph{when should a downstream system trust this formal signal, and when should it treat the signal as out-of-coverage?} The distinction matters because proof failure is not the same as mathematical falsity. A failed attempt may reflect an incorrect answer, but it may also reflect an ill-typed formalization, missing library context, or a proof-search timeout. Treating proof failure as a negative correctness signal therefore risks conflating mathematical incorrectness with lack of formal coverage.

In our pipeline, this asymmetry appears as a sharp reliability gap. When the answer classes holding most of the model's sampled solutions are formally proved, raw formal selection has a low error rate; when few of those solutions land in a proved class, its error rate rises sharply. A second view gives the same message: a proof is reliable when the proved answer clearly outweighs the highest-weight unresolved rival, that is the most-sampled answer class Lean left unresolved, but becomes unsafe when a heavily sampled rival remains unresolved. We refer to this empirical pattern as a \emph{coverage cliff}: formal selection is accurate when proofs cover most of the sampled answer mass and unreliable otherwise.

This motivates a partial-observation view 
of Lean-based answer selection (Figure~\ref{fig:overview}). The system does not observe the truth or falsity of every candidate answer. It observes only which normalized answer classes reach a usable Lean statement and which of those are proved. An \emph{answer class} is a group of candidate strings that denote the same final answer, such as $1/2$, $0.5$, and $2/4$. Candidate classes that are not formalized or not proved are therefore unresolved, not automatically wrong. \methodname{} uses simple diagnostics of this partial observation, that is typed coverage, proved coverage, unresolved rival mass, and formal margin, and accepts a formal answer only when a predeclared rule satisfies a finite-sample selective-risk certificate.

The threshold rule is chosen from a fixed finite grid, where each cell sets one cutoff for typed coverage, proved coverage, and the formal margin. We report two certification regimes. The Bonferroni regime searches over the whole grid on the calibration split and applies a union-bound correction for this search. The dev-then-cal regime chooses the threshold on an independent development split and certifies that single threshold once on the calibration split. Both regimes certify only the accepted formal predictions under their assumptions; neither certifies fallback answers or unresolved
candidates.

\paragraph{Contributions.}
\begin{enumerate}[leftmargin=1.2em,itemsep=0.1em]
    \item \textbf{Coverage cliff.} On MATH-500, the proof-winning answer---the highest-weight answer class with a Lean proof---is correct $96\%$ of the time at high proved coverage but only $20\%$ at low coverage. Hence, a Lean proof is reliable only when coverage and margin are high.
    \item \textbf{A real but sparse and imperfect signal.} The formal signal must be faithfulness-audited, not read as correctness: Lean proves a class for only $28\%$ of problems, and a manual audit finds only $\approx\!43\%$ of those proofs faithful to the problem (the rest are true-but-irrelevant or trivial $\mathtt{X}=\mathtt{X}$ identities).
    \item \textbf{Method.} We consider formal answer selection as a partial-observation selective-prediction problem, propose \methodname{}, which is a risk-controlled selector over coverage diagnostics with finite-sample certificates under two regimes, and prove that unresolved inequivalent answer classes cannot be certified from formal observations alone without abstention.
    \item \textbf{Feasibility is formalizer-governed.} Certificate feasibility is governed by autoformalization coverage: a $7$B formalizer is too sparse for any Bonferroni certificate ($0/20$ bootstrap dev/cal/test splits), whereas a prover-specialized formalizer at $79\%$ coverage flips it to feasible ($17/20$).
\end{enumerate}

\section{Formal Verifiers as NLP Judges}

A general Lean-based judging pipeline has four stages. First, a candidate final answer is extracted from a natural-language solution. Second, the problem and candidate answer are translated into one or more Lean statements. Third, Lean attempts to elaborate each statement. Fourth, proof search attempts to close the elaborated statement. 
A verified proof gives strong positive evidence for the formal statement, but any stage can fail for reasons unrelated to the truth of the original informal answer. Autoformalization may change the semantics of the problem. Elaboration may fail because of syntax, type-class inference, missing imports, or library mismatch. Proof search may time out even for true statements. Hence, the absence of a proof is a heterogeneous event, not a mathematical label.

Prior work demonstrates the usefulness of this pipeline. FANS uses Lean 4 to formalize question--answer pairs and assist answer selection \citep{yao2025fans}. Safe uses Lean 4 for retrospective step-aware verification of natural-language mathematical reasoning \citep{liu2025safe}. These works show that formal proof can improve math reasoning systems. Our question is complementary: how much of the candidate-answer space
is visible to the formal pipeline, and when is that partial trace sufficient for answer selection?

The answer requires separating three notions that are often collapsed: (i) whether a candidate answer is mathematically correct, (ii) whether its intended formal statement is well represented in Lean, and (iii) whether the current prover can decide it within budget. \methodname{} does not solve (i)--(iii). Instead, it measures the observable mass of candidates for which (ii) and (iii) are available, then calibrates the risk of trusting the resulting formal selection.

\section{Problem Setup}
We first describe the objects used by the selector. As a running example, suppose a model samples $K=32$ solutions to the same problem. Several samples may end with surface forms such as $1/2$, $0.5$, and $2/4$; these should count as one final-answer class rather than three different answers. Another group of samples may end with $3/4$. \methodname{} reasons over these normalized classes and their total weights, not over raw strings.

Let $x\in\calX$ be a natural-language mathematical problem with reference answer $a^\star(x)$. A generator or ensemble produces $K$ candidate answers $a_1,\ldots,a_K$ with nonnegative weights $q_j(x)$ satisfying $\sum^K_j q_j(x)=1$. The weights may be self-consistency frequencies \citep{wang2023selfconsistency}, normalized reranker scores, or a uniform distribution over sampled candidates.

\paragraph{Answer classes.}
Raw candidate strings are a poor unit of selection: $1/2$, $0.5$, and $\frac{2}{4}$ may be the same answer, while two candidates with the same surface form may correspond to different formalizations. We therefore map candidates to normalized answer classes
\begin{equation}
    e_x(a_j)\in \calC(x),
\end{equation}
using answer extraction, algebraic simplification, exact string normalization, or formal equivalence checks when available. The class weight is
\begin{equation}
    Q_c(x)=\sum_{j:e_x(a_j)=c} q_j(x),\qquad \sum_{c\in\calC(x)}Q_c(x)=1.
\end{equation}
All coverage, margin, and selection quantities are computed over classes instead of raw candidates. This prevents an unresolved duplicate of an already proved answer from being treated as a rival.

\paragraph{Formal observation.}
For each pair $(x,a_j)$, a formalization module attempts to produce one or more Lean artifacts. A verifier elaborates each artifact and attempts proof search. We summarize the best status for candidate $a_j$ by
\begin{align}
    s_j &\in \calS,\nonumber\\
    \calS &= \{\mathtt{proved},\mathtt{typechecked},\mathtt{illtyped},\nonumber\\
    &\qquad \mathtt{timeout},\mathtt{unformalized}\}.
\end{align}
The five statuses are defined in Appendix~\ref{app:artifact-log} ($\mathtt{proved}$ is a kernel-checked proof; $\mathtt{typechecked}$ elaborates but is unproved; the rest are failures); optional decisive negation/inequivalence statuses are supported but not required.

For each answer class $c$, define class-level indicators
\begin{align}
    P_c(x)&=\ind\{\exists j:e_x(a_j)=c,\ s_j=\mathtt{proved}\},\\
    T_c(x)&=\ind\{\exists j:e_x(a_j)=c,\ s_j\in\calS_{\typ}\},\nonumber\\
    \calS_{\typ}&=\{\mathtt{proved},\mathtt{typechecked},\mathtt{timeout}\}.
\end{align}
$\calS_{\typ}$ is the elaborated-statement set; $\mathtt{timeout}$ qualifies since proof search runs only after elaboration.
The formal observation $\calO(x)$ consists of answer classes, class weights, statuses, proof artifacts when available, error logs, and the fixed prover budget. A selector returns an answer class $\hat c(x)$, a representative answer $\hat a(x)$, or abstains with $\bot$.

\section{Coverage Diagnostics}
The diagnostics answer three questions about the formal trace. How much candidate mass reached a well-formed Lean statement? How much candidate mass was actually proved? And, if a class was proved, does it dominate the strongest unresolved rival?  These quantities are pipeline-agnostic and can be computed from answer-class weights and Lean status logs.
\begin{definition}[Typed coverage]
The typed coverage mass is
\begin{equation}
    C_{\typ}(x)=\sum_{c\in\calC(x)} Q_c(x)T_c(x).
\end{equation}
\end{definition}
Typed coverage measures how much candidate probability reached a well-formed Lean statement. Low typed coverage points to autoformalization, parsing, answer-normalization, or library-interface failure.

\begin{definition}[Proved coverage]
The proved coverage mass is
\begin{equation}
    C_{\prf}(x)=\sum_{c\in\calC(x)} Q_c(x)P_c(x).
\end{equation}
\end{definition}
Proved coverage measures how much answer-class probability is positively decided by the formal pipeline. It is distinct from proof existence: proving a low-weight class may be much less informative than proving the dominant class.

\begin{definition}[Unresolved rival mass and formal margin]
If at least one class is proved, let
\begin{equation}
    c^+(x)=\arg\max_{c:P_c(x)=1} Q_c(x)
\end{equation}
be the highest-weight proved class. The unresolved rival mass is
\begin{equation}
    R_{\unres}(x)=\max_{c\ne c^+(x):P_c(x)=0} Q_c(x),
\end{equation}
with $R_{\unres}(x)=0$ when every other class is also proved or no rival exists. The formal margin is
\begin{equation}
    M(x)=Q_{c^+(x)}(x)-R_{\unres}(x),
\end{equation}
and $M(x)=-\infty$ if no class is proved.
\end{definition}
The margin captures the main reliability condition: a proof of one answer is not decisive if a higher-weight inequivalent answer class remains unresolved.

\paragraph{Conflict indicators.}
If two inequivalent final-answer classes are both proved for the same informal problem, we mark a conflict and reject formal selection by default. This should not be read as a Lean inconsistency. It usually indicates an answer-normalization error, a semantic mismatch between generated formal statements, or incompatible autoformalizations of the informal problem.

\paragraph{Coverage cliffs.}
We use \emph{coverage cliff} for the empirical scenario in which the reliability of formal answer selection changes sharply as a function of $C_{\typ}$, $C_{\prf}$, $R_{\unres}$, or $M$. The hypothesis is stronger than ``hard examples have lower proof rates'': below a coverage threshold, failed or missing proofs become non-diagnostic because unresolved classes may be false, unformalized, or unproved.

\section{Risk-Controlled Formal Selection}

\methodname{} is a selective wrapper around a base Lean selector. It does not require training a new theorem prover or changing the formalizer. Its contribution is to decide when the formal signal is sufficiently covered to determine the answer.

\paragraph{Base formal selector.}
The base selector returns the highest-weight proved answer class,
\begin{equation}
    g_{\mathrm F}(x)=c^+(x),
\end{equation}
with $g_{\mathrm F}(x)=\bot$ if no class is proved or if a conflict is detected. The returned natural-language answer is a canonical representative of $c^+(x)$. More sophisticated formal selectors can be substituted, but the coverage logic is unchanged.

\paragraph{Acceptance rule.}
For a threshold tuple $\tau=(\tau_{\typ},\tau_{\prf},\tau_M)$, define
\begin{align}
A_\tau(x)=\ind\{&C_{\typ}(x)\ge\tau_{\typ},\ C_{\prf}(x)\ge\tau_{\prf},\nonumber\\
&M(x)\ge\tau_M,\ g_{\mathrm F}(x)\ne\bot,\nonumber\\
&\text{no conflict}\}.
\end{align}
If $A_\tau(x)=1$, \methodname{} returns $g_{\mathrm F}(x)$. If $A_\tau(x)=0$, the selective version abstains. A full-coverage system may instead call a fallback $g_{\mathrm N}(x)$ such as self-consistency or a learned natural-language verifier; the formal selective-risk certificate below applies only to the accepted formal predictions, not to fallback outputs. Algorithm~\ref{alg:covcal} (Appendix~\ref{app:method-details}) summarizes the instance-level decision rule.

\paragraph{Risk-controlled threshold selection.}
Given a target selective-risk level $\epsilon$ and a finite grid $\calT$, \methodname{} chooses the least conservative rule whose calibration certificate is valid. We report two valid regimes; both produce a certificate of the form ``with probability $\ge 1-\delta$, the selected $\hat\tau$ has population selective risk at most $\epsilon$,'' but they differ in sample efficiency.

The default operational regime is Bonferroni: optimize over the grid on the calibration split and pay a $|\calT|$-fold union-bound penalty. Formally,
\begin{equation}
    \hat\tau_{\mathrm{Bonf}}=\arg\max_{\tau\in\calT} m_\tau
    \quad\text{s.t.}\quad
    U_{\delta/|\calT|}(k_\tau,m_\tau)\le \epsilon,
    \label{eq:opt}
\end{equation}
where $m_\tau$ and $k_\tau$ are the number of accepted calibration examples and accepted errors, and $U_\alpha(k,m)$ is the one-sided Clopper--Pearson upper bound with $U_\alpha(k,0)=1$. If no threshold is feasible, the selective method returns reject-all. This regime is conservative but makes no assumption on the dependence structure between the $|\calT|$ cell statistics: they are jointly nested under the same calibration sample, so a generic union bound is the safest valid statement.

A tighter alternative is a dev-then-cal rule. First, a threshold
$\hat\tau_{\mathrm{DC}}$ is selected using only the development split, for example by choosing the highest-coverage grid cell whose empirical dev error is at most $\epsilon$ under a deterministic tie-break. If no such cell exists, the procedure returns reject-all. Second, the selected threshold is evaluated once on the independent calibration split and is accepted only when
\begin{equation}
U_\delta\!\left(
k_{\hat\tau_{\mathrm{DC}}}^{\mathrm{cal}},
m_{\hat\tau_{\mathrm{DC}}}^{\mathrm{cal}}
\right)\le \epsilon .
\end{equation}
Because $\hat\tau_{\mathrm{DC}}$ is a function only of the dev split, it is independent of the calibration sample. Hence, a single Clopper--Pearson bound at level $\delta$ is sufficient. Theorem~\ref{thm:devcal} gives the finite-sample guarantee under i.i.d.\ dev/cal/deployment samples.
\begin{table*}[t]
\centering
\small
\setlength{\tabcolsep}{3pt}
\resizebox{0.8\textwidth}{!}{%
\begin{tabular}{lcccc}
\toprule
Method & Overall & Accepted & Acc.\ frac. & Test diag. UB \\
\midrule
\multicolumn{5}{l}{\emph{Regime-independent selectors ($n_{\text{cal}}=151$, $\epsilon=0.15$)}} \\
\midrule
Self-consistency       & 0.910 [0.903, 0.916] & 0.910 [0.903, 0.916] & 1.000 [1.000, 1.000] & 0.138 [0.131, 0.146] \\
Confidence-only abst.  & 0.866 [0.857, 0.875] & 0.969 [0.966, 0.973] & 0.894 [0.884, 0.903] & 0.068 [0.063, 0.072] \\
Raw Lean + fallback    & 0.891 [0.884, 0.899] & 0.891 [0.884, 0.899] & 1.000 [1.000, 1.000] & 0.159 [0.150, 0.168] \\
Proof-existence abst.  & 0.193 [0.182, 0.203] & 0.877 [0.857, 0.896] & 0.220 [0.209, 0.231] & 0.258 [0.234, 0.282] \\
\midrule
\multicolumn{5}{l}{\emph{Calibrated formal selectors --- Bonferroni regime}} \\
\midrule
\textbf{\methodname{}}          & \multicolumn{4}{c}{reject-all ($0/20$ partitions certify)} \\
\midrule
\multicolumn{5}{l}{\emph{Calibrated formal selectors --- dev-then-cal regime ($12/20$ partitions certify)}} \\
\midrule
\textbf{\methodname{}}          & 0.194 [0.184, 0.205] & 0.932 [0.914, 0.949] & 0.209 [0.197, 0.220] & 0.192 [0.168, 0.216] \\
\textbf{\methodname{}+fallback} & 0.905 [0.896, 0.913] & 0.905 [0.896, 0.913] & 1.000 [1.000, 1.000] & 0.144 [0.135, 0.154] \\
\bottomrule
\end{tabular}}
\caption{When \methodname{} accepts ($\approx\!21\%$ of items) it is $0.93$ accurate, but only dev-then-cal certifies this ($12/20$ partitions); Bonferroni rejects the sparse signal outright. As self-consistency alone is already $91\%$ accurate, the contribution is the \emph{certificate}, not raw accuracy. Main MATH-500 run (Qwen2.5-Math-7B-Instruct, $n_{\text{cal}}\!=\!151$, $\epsilon\!=\!0.15$; $K\!=\!20$ bootstrap, seed-mean $95\%$ interval). \emph{Test diag.\ UB} is a held-out Clopper--Pearson bound computed after rule selection---a diagnostic, not the certificate of Table~\ref{tab:regime-comparison}; the top block is regime-independent.}
\label{tab:main}
\end{table*}
\paragraph{Practical choices.}
The default class weights are self-consistency frequencies; a class is marked proved if any of its artifacts is kernel-checked, and two inequivalent proved classes trigger rejection rather than a choice. All implementation choices, and the quantities frozen before calibration labels are used, are listed in Appendix~\ref{app:method-details}.

\section{Finite-Sample Guarantees}
\label{sec:guarantee}
The guarantee certifies the selective risk of the complete frozen pipeline: candidate generation, answer normalization, formalization, proof search, coverage computation, and formal selection. It does not certify unresolved candidates and does not interpret failed Lean proofs as negative labels. 
The probability in the guarantee is over the random draw of calibration and future deployment examples, including any example-level randomness that is part of the frozen pipeline. In the reported bootstrap summaries, candidate and formalization logs are fixed, so only the dev/cal/test partition varies.

\paragraph{Selective risk.}
Let $\calD_{\mathrm{cal}}=\{(X_i,C_i^\star)\}_{i=1}^{n}$ be a calibration set drawn i.i.d. from the same population as future test examples, where $X_i$ is the problem and $C_i^\star$ is its normalized reference answer class. The pipeline outputs used by $A_\tau$ and $g_{\mathrm F}$ are computed without inspecting $C_i^\star$ except for final evaluation. For $\tau\in\calT$, let
\begin{align}
    m_\tau&=\sum_{i=1}^n A_\tau(X_i),\nonumber\\
    k_\tau&=\sum_{i=1}^n A_\tau(X_i)\ind\{g_{\mathrm F}(X_i)\ne C_i^\star\}.
\end{align}
The population selective risk is
\begin{equation}
    R(\tau)=\Pr[g_{\mathrm F}(X)\ne C^\star\mid A_\tau(X)=1],
\end{equation}
with the harmless convention $R(\tau)=0$ when $\Pr[A_\tau(X)=1]=0$.

\begin{theorem}[Uniform selective-risk calibration under Bonferroni]
\label{thm:bonferroni}
Assume the threshold grid $\calT$ is fixed before inspecting calibration labels, and the pipeline used to compute $A_\tau$ and $g_{\mathrm F}$ is fixed before calibration. If the calibration examples are i.i.d. from the deployment population, then with probability at least $1-\delta$ over the calibration sample and any example-level pipeline randomness, for all $\tau\in\calT$,
\begin{equation}
    R(\tau) \le U_{\delta/|\calT|}(k_\tau,m_\tau).
\end{equation}
Consequently, whenever Eq.~\ref{eq:opt} returns a feasible threshold, the selected rule has selective risk at most $\epsilon$ with probability at least $1-\delta$.
\end{theorem}

\noindent Accepted calibration errors at a fixed threshold are Bernoulli, so a one-sided Clopper--Pearson bound \citep{clopper1934use} controls its risk; the Bonferroni correction makes this simultaneous over the grid, letting the threshold be chosen after seeing calibration errors (proof in Appendix~\ref{app:proofs}).

\begin{theorem}[Dev-then-cal selective-risk calibration, no union bound]
\label{thm:devcal}
Partition i.i.d. samples from the population $\mathcal{P}$ into an independent development split $\calD_{\mathrm{dev}}$ and calibration split $\calD_{\mathrm{cal}}$. Let $\hat\tau=f(\calD_{\mathrm{dev}})$ be any $\sigma(\calD_{\mathrm{dev}})$-measurable choice of threshold cell in $\calT\cup\{\bot\}$; in particular, $\hat\tau$ may be the maximum-coverage dev-feasible cell with a deterministic tie-break. Let $m=m_{\hat\tau}^{\mathrm{cal}}$ and $k=k_{\hat\tau}^{\mathrm{cal}}$ be the accept count and accepted-error count on the calibration split. Define the certified output threshold
\[
\widehat\tau_{\mathrm{out}}=\begin{cases}
\bot & \hat\tau=\bot,\ \text{or }m=0,\ \text{or }U_\delta(k,m)>\epsilon,\\
\hat\tau & \text{otherwise.}
\end{cases}
\]
Then
\begin{multline}
\Pr\!\bigl[\widehat\tau_{\mathrm{out}}\neq\bot \ \Longrightarrow\ R(\widehat\tau_{\mathrm{out}})\le U_\delta(k,m)\le\epsilon\bigr] \\ \ge\ 1-\delta.
\end{multline}
\end{theorem}

\noindent Because $\hat\tau$ is fixed on dev, the cal split sees a single threshold, so the bound uses level $\delta$ rather than $\delta/|\calT|$; strict dev/cal separation is the load-bearing assumption (proof in Appendix~\ref{app:proofs}).

\paragraph{dev-then-cal procedure.}
The selection rule, that is the highest-coverage dev-feasible cell then a single cal-side Clopper--Pearson check at level $\delta$, and the boundary cases it handles (empty splits, zero accepts, all-errors, multiple $\epsilon$ targets, and the i.i.d./measurability requirements) are given in Appendix~\ref{app:method-details}.

\paragraph{Comparison of the two regimes.}
At our settings ($|\calT|=125$, $\delta=0.05$), Theorem~\ref{thm:bonferroni} spends $\alpha=4\!\times\!10^{-4}$ per cell while Theorem~\ref{thm:devcal} spends $\alpha=0.05$ on a single cal-side bound, so the two regimes can disagree on feasibility for the same observations. The cliff (Section~\ref{sec:coverage-cliff}) is a property of the data and is identical under either regime; only the certificate's verdict depends on the bound. Section~\ref{sec:ablations} reports both.

\paragraph{No coverage, no verifier-only certificate.}
The next proposition explains why coverage diagnostics are necessary.

\begin{proposition}[Unresolved answer classes are indistinguishable]
\label{prop:indistinguishable}
Fix a problem $x$ and a formal observation $\calO(x)$. Suppose a verifier-only selector $h$ selects a proved class $c_s=h(\calO(x))$, but there is an unresolved answer class $c_u$ with $Q_{c_u}(x)>0$, $P_{c_u}(x)=0$, and no formal evidence in $\calO(x)$ proving, disproving, or identifying $c_u$ as equivalent to $c_s$. Then there are two latent correctness assignments over the informal answer classes that are both compatible with the same formal observation but disagree on whether $c_s$ is correct. Hence no non-vacuous pointwise correctness certificate for $h$ follows from $\calO(x)$ alone unless the selector resolves $c_u$, assumes it away, uses external semantic information or labels, or abstains.
\end{proposition}

Intuitively, the formal log records what Lean could decide, not the truth values of classes it never resolved: an unresolved inequivalent rival leaves the selected proved class and that rival equally compatible with the same observation (proof in Appendix~\ref{app:proofs}). This is a certification statement, not an empirical claim that unresolved candidates are usually correct. Missing formal evidence is not evidence of mathematical falsehood. Calibration turns the partial observation into a controlled selective problem, using Lean where coverage suffices and forcing abstention or fallback elsewhere.

\section{Experimental Protocol \small{\hspace{1em}\codeicon\href{https://github.com/paulinebourigault/covcal-lean}{Code repository}}}
\label{sec:experiments}

We evaluate whether Lean-derived coverage diagnostics predict when formal answer selection is reliable. The main experiment uses a filtered MATH-500 short-answer subset \citep{hendrycks2021math,lightman2024letsverify}: 378 of 500 examples remain after excluding proof-only, diagram-dependent, and non-normalizable references. The split is 76/151/151 development/calibration/test, with $\epsilon=0.15$ and $\delta=0.05$. We also report a harder levels-4/5 subset of the same run, an autoformalizer ablation that swaps Qwen2.5-Coder for Goedel-Prover-V2-8B/-32B with generation held fixed (Sec.~\ref{sec:ablations}), and an AMC/AIME out-of-distribution set (Appendix~\ref{app:additional-results}).

For each problem, Qwen2.5-Math-7B-Instruct generates $K=32$ candidate solutions. We extract final answers, normalize them into answer classes, and assign each class a self-consistency weight $Q_c$. The pipeline then formalizes the top four answer classes and attempts Lean proof search. The resulting class statuses, which are \texttt{proved}, \texttt{typechecked}, \texttt{timeout}, \texttt{illtyped}, and \texttt{unformalized}, are the only formal observations used by \methodname{}.

We compare self-consistency, confidence-only abstention, raw Lean with fallback, proof-existence abstention, and \methodname{} with and without fallback. Thresholds are chosen on the calibration split using either the Bonferroni grid certificate or the dev-then-cal certificate from Sec.~\ref{sec:guarantee}. The main target is $\epsilon=0.15$ with $\delta=0.05$; the harder levels-4/5 subset uses the stricter $\epsilon=0.10$. Prompts (Appendix~\ref{app:prompts}), the candidate/formalization pipeline (Appendix~\ref{app:checklist}), the analysis log schema (Appendix~\ref{app:artifact-log}), and filtering, hardware, and tactic-portfolio details (Appendix~\ref{app:experimental-details}) are in appendices. 

\begin{table}[t]
\centering
\small
\setlength{\tabcolsep}{3pt}
\resizebox{0.8\columnwidth}{!}{%
\begin{tabular}{lccc}
\toprule
Bin & \# & Proof rate & Winner acc. \\
\midrule
\multicolumn{4}{l}{\emph{by proved coverage} $C_{\prf}$ (all $n=378$)} \\
$C_{\prf}<0.25$         & 283 & 0.035 & 0.20 \\
$0.25\le C_{\prf}<0.50$ &   6 & 1.000 & 0.33 \\
$0.50\le C_{\prf}<0.75$ &   4 & 1.000 & 1.00 \\
$C_{\prf}\ge0.75$       &  85 & 1.000 & 0.96 \\
\midrule
\multicolumn{4}{l}{\emph{by formal margin} $M$ (the $105$ proved problems)} \\
$M<0$            & 10 & 1.000 & 0.10 \\
$0\le M<0.25$    &  5 & 1.000 & 0.40 \\
$0.25\le M<0.5$  &  5 & 1.000 & 0.80 \\
$M\ge 0.5$       & 85 & 1.000 & 0.98 \\
\bottomrule
\end{tabular}}
\caption{The coverage/margin cliff on the main run. \emph{\#} is problems per bin; \emph{Proof rate} the fraction with a kernel-proved class; \emph{Winner acc.} the accuracy of the highest-weight proved class among problems with a proof. The signal is sparse (most problems are in the lowest $C_{\prf}$ bin with almost no proofs) but sharply reliable when present: $96\%$ at $C_{\prf}\ge0.75$ vs $20\%$ at $C_{\prf}<0.25$ ($98\%$ vs $10\%$ by margin).}
\label{tab:cliff}
\end{table}
\section{Results}
\label{sec:results}

The results support three main claims. First, the formal signal is real but sparse: Lean proves at least one answer class for only $28\%$ of problems, and a manual audit finds only $\approx\!43\%$ of proved statements faithful to the problem (versus the $74\%$ an automated check labels non-trivial-and-correct; Sec.~\ref{sec:manual-audit}). Second, the signal is sharply coverage-dependent: the highest-weight proved class matches the reference answer $96\%$ of the time at high proved coverage but only $20\%$ at low coverage, and $98\%$ versus $10\%$ binned by formal margin (Tab.~\ref{tab:cliff}). Third, certifiability depends on autoformalization coverage: with the $7$B Qwen2.5-Coder formalizer the sparse signal forces Bonferroni to reject on all $20$ bootstrap partitions (dev-then-cal certifies $12/20$), whereas the prover-specialized Goedel-Prover-V2-8B lifts coverage to $79\%$ and flips Bonferroni to feasible on $17/20$ (Tab.~\ref{tab:main} and~\ref{tab:regime-comparison}). Unless a column says calibration certificate, upper bounds in answer-selection tables are held-out diagnostics computed after rule selection.

\subsection{Main Answer-Selection Results}
\label{sec:coverage-cliff}
On the main run, self-consistency is already $91\%$ accurate with no abstention. The formal signal is far sparser than the candidate signal: proof-existence and the calibrated formal selectors accept only about $20$--$22\%$ of test items, though their accepted accuracy is high ($0.88$--$0.93$). Crucially, the conservative Bonferroni certificate cannot use this sparse signal at all. It returns reject-all on every partition. Thus, a nontrivial accepted set is certified only under the more sample-efficient dev-then-cal regime, and even then on only $12$ of $20$ partitions, with the held-out test diagnostic ($\approx0.19$) sitting above $\epsilon$ because the accepted set is small. The fallback variant recovers self-consistency-level overall accuracy, but the certificate applies only to the formal accepted predictions, not the fallback outputs.

The main diagnostic signal is not that hard examples have fewer proofs. Low-margin examples can have proof existence but still be unreliable because a stronger unresolved rival answer remains. This is the empirical reason for treating unproved candidate mass as unresolved instead of false.
\begin{table*}[t]
\centering\small\setlength{\tabcolsep}{3pt}
\begin{tabular}{lcccccc}
\toprule
\multirow{2}{*}{Autoformalizer (gen.\ fixed)} & \multirow{2}{*}{$n_{\text{cal}}$} & \multicolumn{2}{c}{Bonferroni ($\alpha=\delta/|\calT|$)} & \multicolumn{2}{c}{Dev-then-cal ($\alpha=\delta$)} & UB ratio \\
\cmidrule(lr){3-4} \cmidrule(lr){5-6}
& & reject@$K\!=\!20$ & med cal UB & reject@$K\!=\!20$ & med cal UB & (med) \\
\midrule
Qwen2.5-Coder-7B (main)  & 151 & 20/20 & -- & 8/20 & 0.121 & -- \\
Goedel-Prover-V2-8B      & 151 & 3/20 & 0.127 & 0/20 & 0.062 & 0.49$\times$ \\
Goedel-Prover-V2-32B     & 151 & 7/20 & 0.143 & 3/20 & 0.115 & 0.81$\times$ \\
\bottomrule
\end{tabular}
\caption{The two certificate regimes on the main run ($K\!=\!20$ bootstrap, $\epsilon=0.15$). Bonferroni applies $\alpha=\delta/|\calT|=4\!\times\!10^{-4}$ over the $125$-cell grid; dev-then-cal applies $\alpha=\delta=0.05$ (Theorem~\ref{thm:devcal}). \emph{reject@$K$} counts refusing seeds; \emph{med cal UB} is the median calibration-side bound over feasible seeds (-- if none). Generation is held fixed; the Goedel-Prover-V2 rows vary only the autoformalizer.}
\label{tab:regime-comparison}
\end{table*}
\subsection{Faithfulness Audit and Qualitative Examples}
\label{sec:manual-audit}

An automated rational-evaluation check classifies the $314$ proved artifacts as genuine
($40.1\%$, a $\mathbb{Q}$-true arithmetic identity), structural ($33.4\%$, a true
gcd/divisor/quantified statement), trivial ($22.9\%$, an $\mathtt{X}=\mathtt{X}$ identity), or
spurious ($2.5\%$, $\mathbb{Q}$-false), so $73.6\%$ are non-trivial and correct. But
non-trivial-and-correct is not \emph{faithful}: a true statement may not capture the problem.
A manual audit of $45$ stratified examples (single annotator; diagnostic error analysis) makes
the gap explicit: genuine proofs are faithful ($6/6$) and trivial ones are not, but
\emph{structural} statements almost never are ($0/6$). They assert a true side-fact or bare existence (e.g.\ $\exists n,\ n>125 \wedge n=129$ for a problem whose answer is $129$), omitting the problem's actual condition, while the audited \emph{spurious} cases are classifier
false-positives (e.g.\ $(3{:}\mathbb{Q})/20=0.15$), not real errors. Reweighting the audited
per-category rates by population frequency, only $\approx\!43\%$ of proved statements are
faithful, about half the $73.6\%$ that typecheck and
hold. Faithful verification is thus rarer than proof success suggests, and the audited failures
are heterogeneous, including semantic drift, proof-search failure on faithful statements, and ill-typed
statements, and supporting our treatment of a missing proof as unresolved rather than negative.
Appendix~\ref{app:examples} gives three worked examples (A)--(C): a proof trusted when the
proved class dominates the formally visible mass (A), a missing proof treated as unresolved
rather than wrong (B), and a low-weight or trivially formalized rival correctly distrusted by
the margin (C).

\subsection{Ablations and Certification Regimes}
\label{sec:ablations}

The binding constraint on the main run is autoformalization coverage: only $28\%$ of problems reach any proof (with the $7$B Qwen2.5-Coder autoformalizer), which is what forces the conservative Bonferroni certificate to reject and leaves the tighter dev-then-cal regime feasible on only $12$ of $20$ partitions. Table~\ref{tab:regime-comparison} compares the two regimes on the same frozen observations. Because coverage is the bottleneck, we also vary the autoformalizer, holding generation fixed. The prover-specialized Goedel-Prover-V2-8B raises proved coverage from $28\%$ to $79\%$ at the same proof-winner precision ($86\%$), and this denser signal flips the certificate from infeasible to feasible: Bonferroni now certifies on $17$ of $20$ partitions (versus $0$ with Qwen2.5-Coder) and dev-then-cal on all $20$, each accepting about $48\%$ of test items at $0.98$ accepted accuracy (Table~\ref{tab:regime-comparison}). Scaling the formalizer further to the $32$B Goedel-Prover-V2 does not extend this gain: proved coverage is $64\%$ (below the $8$B's $79\%$) and the certificate is marginally less feasible (Bonferroni $13/20$, dev-then-cal $17/20$), indicating that prover specialization rather than raw model scale is what lifts coverage past the feasibility threshold. The autoformalizer, not the certificate machinery, thus governs whether risk-controlled formal selection is possible, though the faithfulness caveat of Sec.~\ref{sec:manual-audit} still applies to the underlying proofs.
Additional per-level breakdowns, $\epsilon$-sensitivity, status-count, and hard-subset details are reported in Appendix~\ref{app:additional-results}.

\section{Related Work}

\textbf{Math reasoning and verifiers.}
Self-consistency, outcome-reward models, and process-reward models are common signals for natural-language mathematical reasoning \citep{wang2023selfconsistency,cobbe2021gsm8k,uesato2022solving,lightman2024letsverify,wang2024mathshepherd}. Lean-based systems such as FANS and Safe instead use formal proof as a positive correctness signal \citep{yao2025fans,liu2025safe}. Our focus is complementary: when the formal proof is absent, we ask whether the remaining formal trace is sufficiently covered to support answer selection.

\textbf{Autoformalization and theorem proving.}
LLM autoformalization and Lean/Isabelle theorem proving remain sensitive to domain, library context, and dependency retrieval \citep{wu2022autoformalization,jiang2023draftsketchprove,yang2023leandojo,azerbayev2023proofnet,zheng2021minif2f,yu2025formalmath,patel2026mathatlas,lin2025goedel}. These works motivate our treatment of proof failure as heterogeneous instead of as a negative label. We use Qwen2.5-Math for generation and Qwen2.5-Coder for autoformalization \citep{yang2024qwen25math,hui2024qwen25coder}; our contribution is the risk-controlled wrapper.

\textbf{Selective prediction and risk control.}
Selective classification trades coverage for risk \citep{el2010foundations,geifman2017selective}. Risk-controlling prediction and conformal risk-control methods certify population risk for calibrated rules under distributional assumptions \citep{bates2021distribution,angelopoulos2024conformal}. \methodname{} adapts this perspective to formal answer selection with partial observations: unresolved answer mass does not become a larger prediction set, but instead forces abstention unless a coverage rule can be certified.

\section{Conclusion}

We documented a coverage cliff in Lean-as-judge for natural-language mathematical reasoning and proposed \methodname{}, a selective wrapper that certifies risk only where enough formal evidence is visible, treating missing formal evidence as unresolved rather than negative. Whether its finite-sample certificate (Bonferroni or dev-then-cal) holds depends on autoformalization coverage: a $7$B formalizer leaves the signal too sparse for Bonferroni to certify any partition, whereas a prover-specialized formalizer at $79\%$ coverage flips it to feasible on $17$ of $20$. \methodname{} is a conservative interface for using Lean as a judge; feasible when the formalizer covers enough answer mass, but not a guarantee that all answers are verified.

\newpage
\section*{Limitations}
\label{sec:limitations}

\paragraph{Scope of the certificate.}
The selective-risk guarantee covers only examples for which $A_{\hat\tau}(x)\!=\!1$ and a formal answer is returned. It does not certify individual mathematical truth, does not assign a label to unresolved answer classes, and does not extend to fallback predictions made when $A_{\hat\tau}(x)\!=\!0$.

\paragraph{Distributional assumptions.}
Both regimes require exchangeability of the calibration sample with future test points. The dev-then-cal regime additionally requires independence between the dev and cal splits and that $\hat\tau$ be a function of dev only; leakage from the cal split into threshold selection would invalidate Theorem~\ref{thm:devcal}. Distribution shift between calibration and deployment also voids the bound.

\paragraph{Feasibility depends on the autoformalizer.}
The certificate is only as good as the formal coverage it is given. With a $7$B autoformalizer only about $28\%$ of problems are proved, and at $n_{\text{cal}}\!=\!151$ the Bonferroni certificate returns reject-all on every bootstrap partition (dev-then-cal certifies $12/20$, with the held-out diagnostic sometimes exceeding $\epsilon$ on the small accepted set). A prover-specialized formalizer raises coverage to $79\%$ and makes Bonferroni feasible on $17/20$, so the limitation is not the certificate but the formalizer: the method is informative only when autoformalization covers enough answer mass, and the operator must supply a sufficiently specialized formalizer. Raw scale is not a substitute for specialization. The $32$B Goedel-Prover-V2 variant attains lower proved coverage ($64\%$) than the $8$B ($79\%$) and is marginally less feasible (Sec.~\ref{sec:ablations}).

\paragraph{Single dataset and pipeline.}
The empirical evidence is from one dataset family (MATH-500), one candidate generator, and one proof style. We vary the autoformalizer (up to a $32$B prover) but cannot claim the cliff or the coverage numbers transfer to symbolic Olympiad problems, multi-step formal proofs, other answer formats, or systems built on agentic provers with retrieval.

\paragraph{Bootstrap and partition variability.}
The reported $K\!=\!20$ bootstrap captures variance from the dev/cal/test partition but not from candidate sampling or autoformalization stochasticity; the same observations are reused across seeds. A multi-seed re-run of generation and formalization would estimate the latter sources of variability.

\paragraph{Faithfulness of proved statements.}
A kernel-checked proof certifies the \emph{generated} statement, not that the statement faithfully captures the informal problem. An automated rational-evaluation check of the $314$ proved artifacts finds $73.6\%$ non-trivial and correct, but our manual audit ($45$ stratified examples) shows this overstates faithfulness: many true \emph{structural} statements assert a side-fact or bare existence, so only $\approx\!43\%$ of proved statements are faithful to the problem, and a wrong answer class is proved in $8\%$ of problems. This audit is diagnostic error analysis, not benchmark-grade annotation (\textit{e.g.}, it uses a single annotator and small per-category samples) and the automated diagnostics alone do not flag a faithful-looking but semantically drifted statement.

\section*{Ethical Considerations}

Formal verification can make mathematical reasoning systems more reliable, but overclaiming from failed verification can mislabel correct reasoning as incorrect or produce misleading training signals. The proposed calibration framework is intended to make such limitations explicit. The work uses public math datasets and local open models; no personal data is involved. The main risk is that users may interpret calibrated acceptance as a guarantee for all examples rather than for the accepted subset under the stated data assumptions.

\bibliography{refs}

\newpage
\appendix
\section*{Appendices}

\section{Proofs}
\label{app:proofs}

\begin{proof}[Proof of Theorem~\ref{thm:bonferroni}]
Fix a threshold $\tau\in\calT$. Because $A_\tau$ and $g_{\mathrm F}$ are computed without calibration labels, the accepted calibration examples are an i.i.d. sample from the conditional deployment distribution given $A_\tau(X)=1$, conditional on the number $m_\tau$ of accepted examples. Therefore, when $m_\tau>0$, the accepted-error count satisfies
\[
    k_\tau \mid m_\tau \sim \mathrm{Binomial}(m_\tau, R(\tau)).
\]
If $m_\tau=0$, the convention $U_\alpha(k,0)=1$ makes the bound vacuous. For $m_\tau>0$, the one-sided Clopper--Pearson construction gives
\[
    \Pr\!\left[R(\tau)>U_{\delta/|\calT|}(k_\tau,m_\tau)\right] \le \delta/|\calT|.
\]
Taking a union bound over the finite predeclared grid yields
\[
    \Pr\!\left[\exists \tau\in\calT: R(\tau)>U_{\delta/|\calT|}(k_\tau,m_\tau)\right]\le \delta.
\]
On the complementary event, all grid cells are simultaneously valid. Since Eq.~\ref{eq:opt} selects $\hat\tau$ from this same grid after calibration, the selected threshold inherits the bound. If Eq.~\ref{eq:opt} is feasible, then $U_{\delta/|\calT|}(k_{\hat\tau},m_{\hat\tau})\le\epsilon$, hence $R(\hat\tau)\le\epsilon$ on the simultaneous-validity event.
\end{proof}

\begin{proof}[Proof of Theorem~\ref{thm:devcal}]
The map $f$ depends only on the development split, so $\hat\tau$ is $\sigma(\calD_{\mathrm{dev}})$-measurable and independent of the calibration split. Conditional on the event $\hat\tau=\tau^\circ$ and on $m=m_{\tau^\circ}^{\mathrm{cal}}=m^\circ>0$, the cal accepted-error count is binomial:
\[
    k\mid \hat\tau=\tau^\circ,m=m^\circ \sim \mathrm{Binomial}(m^\circ,R(\tau^\circ)).
\]
The one-sided Clopper--Pearson bound therefore satisfies, for every fixed $\tau^\circ$ and $m^\circ>0$,
\[
\Pr\!\left[R(\tau^\circ)>U_\delta(k,m^\circ)\mid \hat\tau=\tau^\circ,m=m^\circ\right] \le \delta.
\]
The certificate outputs $\widehat\tau_{\mathrm{out}}\neq\bot$ only when $m>0$ and $U_\delta(k,m)\le\epsilon$. Thus, conditional on any selected $\hat\tau$ and cal accept count $m>0$, the implication
\[
    \widehat\tau_{\mathrm{out}}\neq\bot \Longrightarrow R(\widehat\tau_{\mathrm{out}})\le U_\delta(k,m)\le\epsilon
\]
fails with probability at most $\delta$. The reject-all branches make the implication true by convention. Marginalizing over $\hat\tau$ and $m$ gives the stated probability bound.
\end{proof}

\begin{proof}[Proof of Proposition~\ref{prop:indistinguishable}]
Because $\calO(x)$ contains no formal decision or equivalence relation for $c_u$, construct two possible latent interpretations that agree on all observed generated artifacts, weights, statuses, proved classes, and error logs. In the first interpretation, the selected class $c_s$ is the intended correct answer class. In the second, the unresolved inequivalent class $c_u$ is the intended correct answer class, while the formal observation remains unchanged because $c_u$ was not resolved and no equivalence to $c_s$ was established. The verifier-only selector receives identical observations in both interpretations and therefore returns the same class $c_s$. It is correct in the first interpretation and incorrect in the second. Thus the formal observation alone cannot yield a pointwise correctness certificate that is valid across all latent interpretations compatible with the unresolved rival.
\end{proof}

\section{Method and Calibration Details}
\label{app:method-details}

\paragraph{Instance-level decision rule.}
\begin{algorithm}[t]
\small
\caption{Coverage-calibrated formal selection (\methodname{}). Lean is trusted only when the proved answer class dominates enough formally visible answer-class mass.}
\label{alg:covcal}
\begin{algorithmic}[1]
\Require problem $x$, candidates $a_{1:K}$, weights $q_{1:K}$, statuses $s_{1:K}$, thresholds $\tau$, fallback $g_{\mathrm N}$
\State Normalize candidates into answer classes $\calC(x)$ and class weights $Q_c(x)$.
\State Aggregate formal statuses into class indicators $T_c(x)$ and $P_c(x)$.
\State Compute $C_{\typ}(x)$, $C_{\prf}(x)$, unresolved rival mass $R_{\unres}(x)$, and margin $M(x)$.
\If{no class is proved or conflicting proved classes exist}
    \State \Return reject formal selection
\EndIf
\If{$C_{\typ}(x)\ge\tau_{\typ}$, $C_{\prf}(x)\ge\tau_{\prf}$, and $M(x)\ge\tau_M$}
    \State \Return the highest-weight proved class $g_{\mathrm F}(x)$
\Else
    \State \Return $\bot$ for selective evaluation, or $g_{\mathrm N}(x)$ for full-answer evaluation
\EndIf
\end{algorithmic}
\end{algorithm}

\paragraph{dev-then-cal selection and certification.}
Define
\[
\calT_{\mathrm{dev}}
=
\left\{
\tau\in\calT:
 m_\tau^{\mathrm{dev}}>0,\;
 k_\tau^{\mathrm{dev}}/m_\tau^{\mathrm{dev}}\le \epsilon
\right\}.
\]
If $\calT_{\mathrm{dev}}=\emptyset$, the procedure returns reject-all.
Otherwise,
\[
\hat\tau_{\mathrm{DC}}
=
\arg\max_{\tau\in\calT_{\mathrm{dev}}} m_\tau^{\mathrm{dev}},
\]
with deterministic lexicographic tie-breaking. This dev-stage screen is not a
certificate; it is only a threshold-selection rule. The selected threshold is
then certified on the independent calibration split and returned only if
\[
m_{\hat\tau_{\mathrm{DC}}}^{\mathrm{cal}}>0
\quad\text{and}\quad
U_\delta\!\left(
k_{\hat\tau_{\mathrm{DC}}}^{\mathrm{cal}},
m_{\hat\tau_{\mathrm{DC}}}^{\mathrm{cal}}
\right)\le \epsilon .
\]
Otherwise the procedure returns reject-all. The requirement
$m_{\hat\tau_{\mathrm{DC}}}^{\mathrm{cal}}>0$ is a practical convention: a
zero-coverage rule has vacuous selective risk, but it is reported as reject-all
because it accepts no formal predictions.

\paragraph{Edge cases handled by Theorem~\ref{thm:devcal}.}
(i) Empty dev split: $\hat\tau=\bot$, reject-all (dev-then-cal is not applicable when the dev fraction is zero).
(ii) No dev-feasible cell: $\hat\tau=\bot$, reject-all.
(iii) Empty cal accept set: $m=0$, reject-all (the $U_\delta(\cdot,0)=1$ convention makes the formula consistent but the certificate explicitly refuses).
(iv) All cal accepts are errors: $k=m$, $U_\delta(m,m)=1$, certificate refuses.
(v) Multiple $\epsilon$ targets on the same data: select $\hat\tau$ once on dev and report the cal-side $U_\delta(k,m)$ as a single number, then check against each $\epsilon$ separately. Re-selecting per $\epsilon$ would re-introduce multiplicity and break the guarantee.
(vi) Dev/cal/test distribution mismatch: the theorem fails. The i.i.d. assumption in the statement is necessary.
(vii) $\hat\tau$ depends on cal: the theorem fails. Strict $\sigma(\calD_{\mathrm{dev}})$-measurability of $\hat\tau$ is the load-bearing assumption.

\paragraph{Practical choices.}
The method is deterministic after candidate generation and formalization logs are fixed. In experiments, we use self-consistency frequency as the primary weight: if $K$ sampled solutions collapse to answer classes $c$, then $Q_c$ is the fraction of samples in class $c$. Reranker-derived weights are reported only as an ablation. A class is marked proved if any of its formal artifacts is kernel checked; it is marked typed if any artifact elaborates or reaches proof search. If one artifact proves and another artifact for the same class fails, the class remains proved but the failure is retained for diagnostics. If two proved classes are inequivalent under answer normalization, the formal selector rejects rather than choosing between them. The threshold grid, target risk $\epsilon$, confidence level $1-\delta$, normalization rules, Lean version, imports, and proof budgets are fixed before calibration labels are used.

\section{Implementation Details}
\label{app:checklist}

The following items document the protocol used in the reported runs:
\begin{enumerate}[leftmargin=1.2em,itemsep=0.2em]
    \item \textbf{Main dataset:} all-levels MATH-500, $378$ kept after filtering ($500$ raw minus $122$ excluded for non-normalizable references, proof-only items, and diagram-dependent geometry).
    \item \textbf{Hard subset:} MATH-500 levels $4$--$5$ only, $199$ kept after the same filters.
    \item \textbf{Candidate generator:} Qwen2.5-Math-7B-Instruct (bfloat16, vLLM 0.10.2), $K=32$ samples, temperature $0.7$, top-$p=0.95$, maximum 2048 new tokens, final answer in \texttt{\\boxed\{...\}}.
    \item \textbf{Class aggregation:} normalize final answers and compute $Q_c$ from self-consistency frequency.
    \item \textbf{Formalization:} Qwen2.5-Coder-7B-Instruct (bfloat16) autoformalization with four artifacts per top class and one error-feedback repair pass.
    \item \textbf{Formalized classes:} top four classes per problem.
    \item \textbf{Lean proving:} Lean $4.21.0$ with Mathlib at commit \texttt{308445d7}; tactic portfolio \texttt{norm\_num}, \texttt{ring\_nf}, \texttt{omega}, \texttt{linarith}, \texttt{nlinarith}, \texttt{simp}, \texttt{aesop}, \texttt{decide}, and short combinations; $20$ seconds per tactic script.
    \item \textbf{Calibration:} 20/40/40 development/calibration/test split with seed $0$; target selective risk $\epsilon=0.15$ on the all-levels run and $\epsilon=0.10$ on the hard subset; confidence $1-\delta=0.95$; Clopper--Pearson bounds with Bonferroni correction over $\calT$ (Appendix~\ref{app:grid}; $|\calT|=125$).
    \item \textbf{Baselines reported:} self-consistency, confidence-only abstention, raw Lean+fallback, proof-existence abstention, typed-coverage only, proved-coverage only, margin-only, \methodname{}, and \methodname{}+fallback.
\end{enumerate}

\section{Threshold Grid Used in Experiments}
\label{app:grid}

The grid used in all reported runs is
\begin{align}
\calT_{\typ} &= \{0,0.25,0.5,0.75,0.9\},\nonumber\\
\calT_{\prf} &= \{0,0.1,0.25,0.5,0.75\},\nonumber\\
\calT_M &= \{-0.5,0,0.1,0.25,0.5\},\nonumber\\
\calT &= \calT_{\typ}\times\calT_{\prf}\times\calT_M .
\end{align}
The grid must be fixed before inspecting calibration labels. Equation~\ref{eq:opt} then selects the accepted-fraction-maximizing rule whose risk upper bound is at most the target $\epsilon$.

\section{Prompt Templates}
\label{app:prompts}

\paragraph{Candidate generation prompt.}
\begin{quote}\small
Solve the following math problem. Give a concise derivation and put the final answer in \texttt{\\boxed\{...\}}. Problem: \emph{[problem text]}
\end{quote}

\paragraph{Lean autoformalization prompt.}
\begin{quote}\small
You are formalizing a math contest answer in Lean 4 with Mathlib. Given the problem and proposed final answer, write a Lean theorem whose proof would certify that the proposed answer is correct for the problem. Requirements: (i) state the problem's actual mathematical claim with the proposed answer substituted in --- do \textbf{not} write a trivial identity such as \texttt{answer = answer}; the theorem must be false if the answer is wrong; (ii) use \texttt{$\mathbb{Q}$} or \texttt{$\mathbb{R}$} (never \texttt{$\mathbb{N}$}) for any division, fraction, or non-integer arithmetic so that division is exact rather than floor division, annotating numeric literals with their type when needed, e.g.\ \texttt{(125 : $\mathbb{Q}$) / 9}; (iii) prefer simple statements and standard Mathlib notation. Return only Lean code. Problem: \emph{[problem text]} Proposed answer: \emph{[answer class]}
\end{quote}

\paragraph{Repair prompt.}
\begin{quote}\small
The Lean code below failed with the following error. Repair the theorem statement or proof while preserving the intended meaning of the problem and answer. Return only Lean code. Code: \emph{[code]} Error: \emph{[Lean error]}
\end{quote}

\section{What Is Generated and Analyzed}
\label{app:artifact-log}

\methodname{} does not require generating a complete formal proof for every problem. It requires a structured log of candidate answers and Lean verification attempts. The minimal unit is an answer class, not a raw model sample. For each problem, the generated record contains:
\begin{enumerate}[leftmargin=1.2em,itemsep=0.15em]
    \item \textbf{Problem fields:} problem id, dataset, topic label, problem text, reference answer, normalized reference class, and inclusion/exclusion status.
    \item \textbf{Candidate fields:} the $K$ raw sampled solutions, extracted final answers, normalized answer classes, and self-consistency weights $Q_c$.
    \item \textbf{Formalization fields:} for each top answer class, the generated Lean code, repair-round index, imports, and theorem statement.
    \item \textbf{Lean fields:} artifact-level status, class-level status, tactic script or proof attempt, runtime, timeout flag, error message, Lean version, Mathlib commit, and hardware.
    \item \textbf{Selection fields:} self-consistency class, highest-weight proved class, conflict flag, $C_{\typ}$, $C_{\prf}$, $R_{\unres}$, $M$, \methodname{} accept/reject decision, fallback decision when used, and correctness.
\end{enumerate}

\paragraph{Per-status definitions.}
The best Lean status $s_j$ of an artifact takes one of five values. $\mathtt{proved}$ means a Lean proof of the generated statement was found and kernel checked. $\mathtt{typechecked}$ means the statement elaborated but no proof was found. $\mathtt{timeout}$ means proof search began from a well-formed statement but exceeded the budget. $\mathtt{illtyped}$ and $\mathtt{unformalized}$ denote earlier failures. If a pipeline can also prove negations or inequivalence, those outcomes can be added as optional decisive statuses; they are not required by our method.

\paragraph{Existing verified Lean proofs.}
Existing proof corpora can be used for supplementary analysis of proof-search or proof-edit coverage, but they do not replace the main answer-selection experiment. \methodname{} requires multiple candidate final-answer classes with weights and correctness labels. A corpus of already verified Lean proofs usually supplies proof success but not alternative wrong answer classes, unresolved rival mass, or natural-language answer-selection labels.

\section{Additional Experimental Details}
\label{app:experimental-details}

\paragraph{Filtering and splits.}
We include problems whose final answer can be extracted as a number, expression, set, interval, finite tuple, or multiple-choice label mapped to a mathematical answer. We exclude diagram-dependent geometry, proof-only items, and examples whose official answer cannot be normalized. MATH-500 contributes $378$ kept examples after $122$ exclusions; the hard levels-4/5 subset is the $199$-example slice of these at levels 4 and 5. Unless otherwise stated, examples are split after generation and formalization into development, calibration, and test partitions. The seed-0 split is preregistered; the bootstrap summaries resample the frozen observations over $K=20$ dev/cal/test partitions.

\paragraph{Generation and normalization.}
Qwen2.5-Math-7B-Instruct generates $32$ solutions per problem at temperature $0.7$, top-$p=0.95$, and maximum $2048$ new tokens. The normalizer strips admissible units and formatting, canonicalizes fractions/decimals/signs/whitespace, parses simple symbolic expressions, tests numeric equivalence for rationals, decimals, radicals, and finite tuples, and logs ambiguous cases. Class weights are self-consistency frequencies over normalized answer classes.

\paragraph{Formalization and Lean proving.}
For each of the top four answer classes, Qwen2.5-Coder-7B-Instruct produces four Lean artifacts, followed by one error-feedback repair pass when elaboration fails. Lean proof search uses Lean $4.21.0$ with Mathlib commit \texttt{308445d7}; the tactic portfolio is \texttt{norm\_num}, \texttt{ring\_nf}, \texttt{omega}, \texttt{linarith}, \texttt{nlinarith}, \texttt{simp}, \texttt{aesop}, \texttt{decide}, and short combinations, with a $20$-second per-script budget. Main-run class-level statuses are $140$ proved, $271$ typechecked, $251$ illtyped, $0$ timeout, and $741$ unformalized over $1403$ answer-class observations, indicating that autoformalization-tier coverage, not prover timeout, dominates the negative trace.

\paragraph{Baselines.}
The reported baselines are self-consistency, confidence-only abstention, raw Lean plus fallback, proof-existence abstention, typed-coverage-only, proved-coverage-only, margin-only, \methodname{}, and \methodname{} plus fallback. The fallback is self-consistency. Fallback predictions are useful for overall accuracy but are not covered by the selective-risk certificate unless separately calibrated.

\paragraph{Compute and hardware.}
All experiments ran on a single compute node with one NVIDIA H100 80GB HBM3 GPU and $8$ allocated CPU cores; generation and autoformalization use vLLM on the GPU, while Lean proof search runs on the CPU within the same job. The autoformalizer ablations took approximately $5.5$ (Goedel-Prover-V2-8B) and $4.9$ (Goedel-Prover-V2-32B) GPU-hours, and the main MATH-500 and AMC/AIME runs a few GPU-hours each (generation is replayed from cache across the ablations), for roughly $15$--$20$ GPU-hours in total across all reported runs.

\section{Additional Results}
\label{app:additional-results}

This appendix presents supporting evidence behind the results
summary in Sec.~\ref{sec:results}. Tables~\ref{tab:app-per-level}--\ref{tab:app-hard-amc}
provide the corresponding per-level, sensitivity, failure-status, and
robustness details.

\subsection{Full Answer-Selection Comparison}
\label{app:full-selectors}
Table~\ref{tab:app-selectors} reports all nine selectors on the main run under the dev-then-cal regime (under Bonferroni the calibrated rows are reject-all; Table~\ref{tab:regime-comparison}). The single-diagnostic selectors clarify which diagnostic carries the signal: margin-only matches the joint \methodname{} rule exactly ($0.93$ accepted accuracy at $0.21$ accepted fraction), whereas typed- and proved-coverage alone are slightly less precise. Thus, the formal margin carries most of the selective signal at this protocol. Figure~\ref{fig:selective-risk} plots the accepted-fraction/risk tradeoff for the main selectors.

\begin{table}[t]
\centering\small\setlength{\tabcolsep}{3pt}
\resizebox{\columnwidth}{!}{%
\begin{tabular}{lcccc}
\toprule
Selector & Overall & Accepted & Acc.\ frac. & Test UB \\
\midrule
Self-consistency       & 0.910 & 0.910 & 1.000 & 0.138 \\
Confidence-only        & 0.866 & 0.969 & 0.894 & 0.068 \\
Raw Lean + fallback    & 0.891 & 0.891 & 1.000 & 0.159 \\
Proof-existence        & 0.193 & 0.877 & 0.220 & 0.258 \\
Typed-coverage only    & 0.194 & 0.867 & 0.224 & 0.268 \\
Proved-coverage only   & 0.194 & 0.874 & 0.222 & 0.258 \\
Margin only            & 0.194 & 0.932 & 0.209 & 0.192 \\
\methodname{}          & 0.194 & 0.932 & 0.209 & 0.192 \\
\methodname{}+fallback & 0.905 & 0.905 & 1.000 & 0.144 \\
\bottomrule
\end{tabular}}
\caption{All nine selectors on the main run, dev-then-cal regime ($K\!=\!20$ bootstrap, $n_{\text{cal}}=151$, $\epsilon=0.15$; seed-means). \emph{Overall} counts abstentions as errors; \emph{Acc.\ frac.} is the accepted fraction; \emph{Test UB} is the held-out diagnostic upper bound. The calibrated rows (typed/proved/margin/\methodname{}) average over the $12/20$ feasible seeds; the others are regime-independent. Margin-only and \methodname{} coincide, indicating the margin threshold carries most of the selective signal here.}
\label{tab:app-selectors}
\end{table}

\begin{figure}[t]
\centering
\includegraphics[width=0.75\columnwidth]{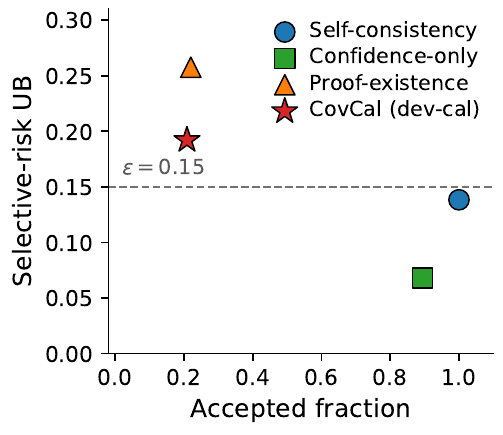}
\caption{Accepted fraction versus held-out selective-risk upper bound on the main run (dev-then-cal; $K\!=\!20$ bootstrap mean). The dashed line is $\epsilon=0.15$. The formal selectors accept $\approx21\%$ of items just above $\epsilon$, while confidence-only abstention reaches $0.89$ coverage at lower diagnostic risk; \mbox{+fallback} variants (full coverage) and the Bonferroni reject-all case are omitted.}
\label{fig:selective-risk}
\end{figure}

\subsection{Per-Level and Domain Breakdowns}
\label{app:per-level-results}

Table~\ref{tab:cliff} bins the cliff coarsely; Table~\ref{tab:app-per-level} aggregates the same main-run observations by MATH-500 difficulty level. Both coverage and reliability decline monotonically with difficulty: mean proved coverage falls from $0.44$ at level~1 to $0.16$ at level~5, and proof-existence accuracy on the proved subset falls from $1.00$ to $0.74$. The coverage drop is the steeper of the two (a $2.7\times$ reduction versus $1.4\times$ for accuracy), so in this protocol increasing difficulty primarily reduces visible formal coverage, while in-bin reliability degrades more slowly.

\begin{table}[t]
\centering
\small
\setlength{\tabcolsep}{3pt}
\resizebox{0.9\columnwidth}{!}{%
\begin{tabular}{cccccc}
\toprule
Level & $n$ & $\bar C_{\typ}$ & $\bar C_{\prf}$ & PE acc. & Cert. feas. \\
\midrule
1 & 35  & 0.797 & 0.441 & 1.00 & -- \\
2 & 64  & 0.733 & 0.328 & 0.91 & -- \\
3 & 80  & 0.627 & 0.226 & 0.90 & -- \\
4 & 97  & 0.616 & 0.191 & 0.78 & -- \\
5 & 102 & 0.520 & 0.162 & 0.74 & -- \\
\bottomrule
\end{tabular}}
\caption{Per-level breakdown on the main MATH-500 run. $\bar C_{\typ}$ and $\bar C_{\prf}$ are mean typed/proved coverage within each level. PE acc. is proof-existence accepted accuracy on examples where at least one answer class is proved. Cert. feas. asks whether proof-existence alone certifies $\epsilon=0.15$ within that level using the same Bonferroni confidence budget; all individual levels are too small for this single-bin stand-in to certify, which is expected and should not be read as evidence of zero useful signal.}
\label{tab:app-per-level}
\end{table}

Table~\ref{tab:app-domain} gives a domain-level view. Geometry and calculus/analysis have substantially lower typed and proved coverage than algebra, number theory, and combinatorics (proved coverage $0.19$ and $0.05$ versus $0.23$--$0.31$), so far fewer of their problems reach any formal verdict. This supports this work's central interpretation that the reliability of Lean-as-judge depends on how much answer-class mass becomes representable and
decidable by the formal pipeline.

\begin{table}[t]
\centering
\small
\setlength{\tabcolsep}{3pt}
\resizebox{\columnwidth}{!}{%
\begin{tabular}{lccccc}
\toprule
Domain & $n$ & Typed & Proved & Acc. & Top failure \\
\midrule
Algebra              & 241 & 0.663 & 0.234 & 0.89 & illtyped \\
Number theory        &  60 & 0.684 & 0.314 & 0.82 & illtyped \\
Combinatorics        &  30 & 0.621 & 0.281 & 0.70 & illtyped \\
Geometry             &  24 & 0.490 & 0.194 & 0.86 & illtyped \\
Calculus / analysis  &  23 & 0.281 & 0.053 & 1.00 & illtyped \\
\bottomrule
\end{tabular}}
\caption{Domain-level formal coverage on the main run. Typed and Proved are mean $C_{\typ}$ and $C_{\prf}$ per domain; Acc. is the accuracy of the highest-weight proved class among problems where a proof exists; Top failure is the modal class-level Lean status among unsuccessful artifacts. Coverage is markedly lower for geometry and calculus/analysis (typed $0.49$ and $0.28$ versus $0.66$ for algebra); the Acc. column is over the small proved subset of each domain (e.g.\ only $2$ proved problems in calculus/analysis) and is noisy.}
\label{tab:app-domain}
\end{table}

\subsection{Status Counts and \texorpdfstring{$\epsilon$}{epsilon}-Sensitivity}
\label{app:status-eps}

Table~\ref{tab:app-status} gives the automated Lean-status breakdown. The dominant statuses are \texttt{unformalized} and \texttt{illtyped}; timeouts on well-formed statements are absent in the main run. These status counts support the claim that the main bottleneck occurs before proof search, but they remain coarser than a manual semantic audit.

\begin{table}[t]
\centering
\small
\setlength{\tabcolsep}{3pt}
\resizebox{0.75\columnwidth}{!}{%
\begin{tabular}{lcc}
\toprule
Lean class-level status & Count & Share \\
\midrule
\texttt{unformalized}  & 741 & 0.528 \\
\texttt{typechecked}   & 271 & 0.193 \\
\texttt{illtyped}      & 251 & 0.179 \\
\texttt{proved}        & 140 & 0.100 \\
\texttt{timeout}       &   0 & 0.000 \\
\bottomrule
\end{tabular}}
\caption{Class-level Lean status distribution on the main run ($1403$ answer-class
observations across $378$ problems; best status per class). The taxonomy is
automatic and based on pipeline statuses. Most answer classes are never
formalized or fail elaboration; only $10\%$ are proved, and no well-formed
statement times out.}
\label{tab:app-status}
\end{table}

\paragraph{Why the $32$B autoformalizer does not help (status-level view).}
Table~\ref{tab:app-status-ablation} compares class-level Lean statuses for the two
Goedel-Prover-V2 autoformalizers with generation held fixed, and explains the headline
of Sec.~\ref{sec:ablations}, that scaling from $8$B to $32$B lowers proved
coverage ($79\%\!\to\!64\%$). The gap is not a proof-search timeout: \texttt{timeout} is
$0$ for both models.
Both models also elaborate statements at essentially the same rate. The
\texttt{unformalized} count is nearly identical ($656$ vs.\ $651$ of $1380$ classes).
The entire difference is a \texttt{proved}$\to$\texttt{typechecked} shift: $493$ classes
are proved by the $8$B but only $347$ by the $32$B, with the lost mass reappearing as
\texttt{typechecked} ($191\!\to\!323$). The larger model writes statements that type-check
just as often but that the fixed Mathlib tactic portfolio closes less often, consistent
with prover specialization (shared by both Goedel-Prover-V2 models), rather than raw
parameter count, being what lifts proved coverage past the certificate-feasibility threshold.

\begin{table}[t]
\centering
\small
\setlength{\tabcolsep}{4pt}
\resizebox{\columnwidth}{!}{%
\begin{tabular}{lcc}
\toprule
Lean class-level status & Goedel-V2-8B & Goedel-V2-32B \\
\midrule
\texttt{proved}        & 493 & 347 \\
\texttt{typechecked}   & 191 & 323 \\
\texttt{illtyped}      &  40 &  59 \\
\texttt{unformalized}  & 656 & 651 \\
\texttt{timeout}       &   0 &   0 \\
\midrule
Proved coverage (problems) & $79.1\%$ & $63.5\%$ \\
\bottomrule
\end{tabular}}
\caption{Class-level Lean status counts for the two autoformalizer ablations ($1380$
answer-class observations over the same $378$ problems; best status per class, generation
held fixed). Scaling $8$B$\to$$32$B leaves \texttt{unformalized} and \texttt{timeout}
essentially unchanged but shifts mass from \texttt{proved} to \texttt{typechecked}: the
larger model's statements elaborate as often but are closed by the fixed tactic portfolio
less often, so proved coverage falls from $79\%$ to $64\%$.}
\label{tab:app-status-ablation}
\end{table}

Table~\ref{tab:app-eps-sensitivity} reports the Bonferroni calibration certificates on the seed-0 main split. This table is intentionally separated from held-out answer-selection diagnostics: it is the calibration-side certificate used by the selection rule. On the sparse real signal the main split returns reject-all through $\epsilon=0.20$ and certifies only at $\epsilon\gtrsim0.26$; the smallest achievable Bonferroni upper bound at the most permissive cell is $0.257$, well above the preregistered target $\epsilon=0.15$.

\begin{table}[t]
\centering
\small
\setlength{\tabcolsep}{3pt}
\resizebox{0.75\columnwidth}{!}{%
\begin{tabular}{lccc}
\toprule
$\epsilon$ & Selected $\hat\tau$ & $(m_{\hat\tau},k_{\hat\tau})$ & Cal. UB \\
\midrule
0.05 & reject-all     & n/a     & n/a \\
0.10 & reject-all     & n/a     & n/a \\
0.15 & reject-all     & n/a     & n/a \\
0.20 & reject-all     & n/a     & n/a \\
0.30 & $(0,0.1,0)$    & $(35,1)$ & 0.257 \\
\bottomrule
\end{tabular}}
\caption{Bonferroni calibration-split certificates for the seed-0 main run ($n_{\mathrm{cal}}=151$, $\delta=0.05$, $|\mathcal T|=125$). On the sparse real signal no grid cell certifies through $\epsilon=0.20$, so \methodname{} returns reject-all at the preregistered target $\epsilon=0.15$; the most permissive cell first certifies near $\epsilon=0.26$ with calibration upper bound $0.257$. This is the conservative regime correctly refusing when the proved signal is too sparse to pay the $|\mathcal T|=125$ union bound.}
\label{tab:app-eps-sensitivity}
\end{table}
A coarse-grid sensitivity check gives the same qualitative conclusion: a $3\times3\times3$ grid lowers the most-permissive-cell upper bound only modestly and does not bring it below $\epsilon=0.15$. Thus the reject-all verdict is driven by the sparsity of the proved signal and the resulting accepted-set errors, not by the granularity of the threshold grid.

\subsection{Hard Subset and AMC/AIME}
\label{app:hard-amc}

Table~\ref{tab:app-hard-amc} records two robustness settings. The hard MATH-500 levels~4--5 subset ($n=199$, the hardest slice of the main run) has an even sparser proved signal than the full set and returns reject-all under \emph{both} regimes at the stricter target $\epsilon=0.10$. An AMC/AIME-style out-of-distribution set ($n=164$), run on the same pipeline, confirms that the sparsity persists under a harder problem distribution: self-consistency accuracy is far lower ($0.40$ versus $0.91$ on MATH-500), the proved signal is sparse (proof-existence accepts only $\approx\!11\%$ of test items), and both certificate regimes return reject-all at $\epsilon=0.10$.

\begin{table*}[t]
\centering
\small
\setlength{\tabcolsep}{3pt}
\resizebox{0.75\textwidth}{!}{%
\begin{tabular}{lcccccc}
\toprule
Setting & $n$ & $n_{\rm cal}$ & $\epsilon$ & SC & PE acc. & PE frac. / verdict \\
\midrule
Main MATH-500       & 378 & 151 & 0.15 & 0.910 & 0.877 & 0.220 / dev-cal $12/20$ \\
Hard MATH-500 L4--5 & 199 &  80 & 0.10 & 0.859 & 0.761 & 0.231 / reject-all (both) \\
AMC/AIME            & 164 &  66 & 0.10 & 0.400 & 0.429 & 0.108 / reject-all (both) \\
\bottomrule
\end{tabular}}
\caption{Robustness settings versus the full main run. SC is self-consistency test accuracy; PE acc.\ and PE frac.\ are proof-existence accepted accuracy and accepted fraction; the last field gives the certificate verdict ($\epsilon$ as listed). The hard subset is the level-4/5 slice of the main run ($n=199$); its proved signal is even sparser, so both regimes reject-all at the stricter $\epsilon=0.10$. The AMC/AIME row is an out-of-distribution check on the same pipeline; its sparser signal and lower self-consistency accuracy likewise yield reject-all under both regimes at $\epsilon=0.10$.}
\label{tab:app-hard-amc}
\end{table*}

For the hard subset, confidence-only abstention can produce a low held-out error bound, but that is not a Lean-backed formal certificate. The formal coverage-based rows reject because the calibration-side Bonferroni upper bound for proof-existence alone is already above the target. This is the intended conservative behavior of the certificate when formal coverage is too low for the
specified risk level.

\section{Worked Examples}
\label{app:examples}

The three cases of Section~\ref{sec:manual-audit} are presented next.

\paragraph{(A) Accept --- faithful proof, full coverage.}
\textcolor{red}{\emph{Problem:}} ``What is the $2003$rd term of the sequence of odd numbers $1, 3, 5, 7, \dots$?'' \textcolor{blue}{\emph{Reference answer:}} $4005$. All $K{=}32$ samples collapse to the single answer class $4005$ ($w{=}1.0$), which is proved. The kernel-checked statement is
\[
2\cdot 2003 - 1 = 4005 \qquad (\text{by } \mathtt{norm\_num}),
\]
a faithful closed form of ``the $n$th odd number is $2n-1$.'' Diagnostics: $C_{\typ}{=}C_{\prf}{=}1$, margin $1$. \methodname{} accepts $4005$ (correct). \textcolor{orange}{\emph{Lesson:}} a proof is trusted when the proved class dominates the formally visible mass with a clear margin.

\paragraph{(B) Abstain --- no formal signal, fallback correct.}
\textcolor{red}{\emph{Problem:}} ``Simplify $\cos\!\frac{2\pi}{15}\,\cos\!\frac{4\pi}{15}\,\cos\!\frac{8\pi}{15}\,\cos\!\frac{16\pi}{15}$.'' \textcolor{blue}{\emph{Reference answer:}} $\tfrac{1}{16}$. The dominant class $1/16$ ($w{=}0.84$) and the minor classes $1/32$ and $1/8$ all fail Lean elaboration. The trigonometric product identity does not autoformalize. So, $C_{\typ}{=}C_{\prf}{=}0$ and no class is proved. \methodname{} abstains; the self-consistency fallback returns $1/16$ (correct). \textcolor{orange}{\emph{Lesson:}} a missing proof is unresolved, not negative; abstention plus fallback preserves the correct answer.

\paragraph{(C) Abstain --- margin distrusts a vacuous proof of a wrong rival.}
\textcolor{red}{\emph{Problem:}} ``An equilateral triangle is inscribed in the parabola $x^2 = 8y$, such that one of the vertices of the triangle coincides with the vertex of the parabola. Find the side length of this equilateral triangle.'' \textcolor{blue}{\emph{Reference answer:}} $16\sqrt{3}$. The correct class $16\sqrt{3}$ ($w{=}0.88$) fails to autoformalize (illtyped); the only proved class is a low-weight rival $16$ ($w{=}0.03$), whose ``proof'' is the vacuous identity
\[
(16 : \mathbb{Q}) = 16 \qquad (\text{by } \mathtt{norm\_num}),
\]
which holds for any value and therefore provides no evidence about the answer. Diagnostics: $C_{\prf}{=}0.03$, margin $-0.84$. Because the proved class carries far less weight than the unresolved dominant class, the strongly negative margin makes \methodname{} abstain; the fallback recovers $16\sqrt{3}$ (correct). \textcolor{orange}{\emph{Lesson:}} proof existence on a low-weight or trivially formalized rival is correctly distrusted by the margin.

\end{document}